\documentclass[conference]{IEEEtran}
\IEEEoverridecommandlockouts

\usepackage[utf8]{inputenc}

\usepackage[square,numbers]{natbib}

\usepackage{algorithmic}
\usepackage{textcomp}

\usepackage[usenames, dvipsnames, svgnames,table]{xcolor}
\definecolor{lightblue}{RGB}{60,60,200}
\definecolor{darkblue}{RGB}{1,1,255}
\definecolor{mplblue}{RGB}{31, 119, 180}
\definecolor{cskyblue}{rgb}{0.01,0.39,0.75}

\def\BibTeX{{\rm B\kern-.05em{\sc i\kern-.025em b}\kern-.08em
    T\kern-.1667em\lower.7ex\hbox{E}\kern-.125emX}}

\usepackage[hidelinks]{hyperref}
\hypersetup{
    linkcolor  = cskyblue,
    citecolor  = cskyblue,
    urlcolor   = cskyblue,
    colorlinks = true,
    linktoc=all,
}
\usepackage{url}            
\usepackage{booktabs}       
\usepackage{microtype}      

\usepackage{amsmath}
\usepackage{amssymb}
\usepackage{amsthm}
\usepackage{amsfonts}
\usepackage{bbm}
\usepackage{mathtools}
\usepackage{nicefrac}

\newtheorem{theorem}{Theorem}[section]
\newtheorem{lemma}{Lemma}

\newtheorem{corollary}{Corollary}
\newtheorem{remark}{Remark}

\usepackage{multirow}
\usepackage{array}

\usepackage{subcaption}

\usepackage{graphicx} 
\usepackage{tikz}
\usetikzlibrary{arrows.meta,calc,decorations.markings,math,arrows.meta}

\newcommand{\rbr}[1]{\left(#1\right)}
\newcommand{\sbr}[1]{\left[#1\right]}
\newcommand{\cbr}[1]{\left\{#1\right\}}

\newcommand{\abs}[1]{\left|#1\right|}

\newcommand{\mathset}[1]{\cbr{#1}}  

\newcommand\numberthis{\addtocounter{equation}{1}\tag{\theequation}}  

\newcommand{\der}[2]{\frac{\text{d} #1}{\text{d} #2} }

\DeclareMathOperator{\E}{\mathbb{E}}

\newcommand{\ignore}[1]{}

\def \ZZ {\mathcal{Z}}

\def \WW {\mathcal{W}}

\def \bR {\mathbb{R}}

\newcommand{\KL}[2]{\text{KL}\rbr{#1\ \lVert\ #2}}

\usepackage[capitalize,nameinlink,noabbrev]{cleveref}
\crefformat{equation}{(#2#1#3)}
\crefmultiformat{equation}{(#2#1#3)}%
{ and~(#2#1#3)}{, (#2#1#3)}{ and~(#2#1#3)}
\crefformat{theorem}{Thm.~#2#1#3}
\crefformat{lemma}{Lemma~#2#1#3}
\crefformat{proposition}{Proposition~#2#1#3}
\crefformat{corollary}{Corollary~#2#1#3}
\crefformat{remark}{Rem.~#2#1#3}
\crefformat{section}{Sec.~#2#1#3}
\crefformat{example}{Example~#2#1#3}
\crefformat{figure}{Fig.~#2#1#3}
\crefformat{algorithm}{Algorithm~#2#1#3}
\crefrangeformat{section}{Sec.~#3#1#4--#5#2#6}
\crefrangeformat{equation}{~(#3#1#4--#5#2#6)}
\crefrangeformat{figure}{Fig.~#3#1#4--#5#2#6}
\crefformat{appendix}{Appendix~#2#1#3}

\newcommand{\Choose}[2]{\rbr{\begin{matrix}#1\\#2\end{matrix}}}
\newcommand{\floor}[1]{\lfloor #1 \rfloor}

\begin{document}

\title{Formal limitations of sample-wise information-theoretic generalization bounds\\
\thanks{H.H. was supported by a USC Annenberg Fellowship.}
}

\author{\IEEEauthorblockN{Hrayr Harutyunyan, Greg Ver Steeg, and Aram Galstyan}
\IEEEauthorblockA{\textit{USC Information Sciences Institute} \\
\texttt{\{hrayrh, gregv, galstyan\}@isi.edu}
}
}

\maketitle

\begin{abstract}
Some of the tightest information-theoretic generalization bounds depend on the average information between the learned hypothesis and a \emph{single} training example.
However, these sample-wise bounds were derived only for \emph{expected} generalization gap.
We show that even for expected \emph{squared} generalization gap no such sample-wise information-theoretic bounds exist.
The same is true for PAC-Bayes and single-draw bounds.
Remarkably, PAC-Bayes, single-draw and expected squared generalization gap bounds that depend on information in pairs of examples exist.
\end{abstract}

\begin{IEEEkeywords}
learning theory, generalization bounds, information theory
\end{IEEEkeywords}

\section{Introduction}\label{sec:intro}
Modern deep learning models have excellent generalization capabilities despite being highly overparameterized. 
\citet{zhang2016understanding} demonstrated that generalization bounds based solely on a notion of complexity of hypothesis class fail to explain this phenomenon -- a learner with a particular hypothesis class may generalize well for one data distribution, but fail to generalize for another data distribution.
This has sparked a search for data-dependent and algorithm-dependent generalization bounds that give non-vacuous results for large neural networks~\citep{Jiang*2020Fantastic,dziugaite2020search}.

Prominent data and algorithm dependent bounds include PAC-Bayes and information-theoretic generalization bounds.
PAC-Bayes bounds are usually based on the Kullback-Leilber divergence from the distribution of hypotheses after training (the ``posterior'' distribution) to a fixed ``prior'' distribution~\citep{ShaweTaylor1997APA,mcallester1999some, mcallester1999pac, catonibook2007, alquier2021user}.
Information-theoretic bounds are based on various notions of training set information captured by the training algorithm~\citep{xu2017information, bassily2018learners, negrea2019information, bu2020tightening, steinke2020reasoning, haghifam2020sharpened, neu2021information,raginsky202110,hellstrom2020generalization, harutyunyan2021informationtheoretic, esposito2021generalization}.
For both types of bounds the main conclusion is that when a training algorithm captures little information about the training data then the generalization gap should be small.
Recently, a few works demonstrated that some of the best PAC-Bayes and information-theoretic generalization bounds produce non-vacuous results in practical settings~\citep{DBLP:conf/uai/DziugaiteR17,perez2021tighter,harutyunyan2021informationtheoretic}.

A key ingredient in recent improvements of information-theoretic generalization bounds was the introduction of sample-wise information bounds by \citet{bu2020tightening}, where one measures how much information on average the learned hypothesis has about a single training example.
While PAC-Bayes and information-theoretic bounds are intimately connected to each other, the technique of measuring information with single examples 
has not appeared in PAC-Bayes bounds. 
This paper explain the curious omission of single example PAC-Bayes bounds by proving the non-existence of such bounds, revealing a striking difference between information-theoretic and PAC-Bayesian perspectives. 
The reason for this difference is that PAC-Bayes upper bounds the probability of average population and empirical risks being far from each other, while information-theoretic generalization methods upper bound expected difference of population and empirical risks, which is an easier task.

\subsection{Learning setting}
Consider a standard learning setting. There is an unknown data distribution $P_Z$ on an input space $\ZZ$.
The learner observes a collection of $n$ i.i.d examples $S=(Z_1,\ldots,Z_n)$ sampled from $P_Z$ and outputs a hypothesis (possibly random) belonging to a hypothesis space $\mathcal{W}$.
We will treat the learning algorithm as a probability kernel $Q_{W|S}$, which given a training set $s$ outputs a hypothesis sampled from the distribution $Q_{W|S=s}$.
Together with $P_S$, this induces a joint probability distribution $P_{W,S}=P_S Q_{W|S}$ on $\mathcal{W} \times \mathcal{Z}^n$.
The performance of a hypothesis $w\in\mathcal{W}$ on an example $z\in\ZZ$ is measured with a loss function $\ell : \mathcal{W} \times \mathcal{Z} \rightarrow \bR$.
For a hypothesis $w\in\mathcal{W}$, the population risk $R(w)$ is defined as $\E_{Z'\sim P_Z}\sbr{\ell(w,Z')}$, while the empirical risk is defined as $r_S(w) = 1/n \sum_{i=1}^n \ell(w,Z_i)$.
In this setting one can consider different types of generalization bounds.

\paragraph{Expected generalization gap bounds}
Let $W \sim Q_{W|S}$ be the output of the training algorithm on training set $S$.
The simplest quantity to consider is the expected generalization gap: $\abs{\E_{P_{W,S}}\sbr{R(W) - r_S(W)}}$.
\citet{xu2017information} showed that if $\ell(w, Z')$ with $Z'\sim P_Z$ is $\sigma$-subgaussian for all $w\in\mathcal{W}$, then $\abs{\E_{P_{W,S}}\sbr{R(W) - r_S(W)}} \le \sqrt{\frac{2\sigma^2 I(W;S)}{n}}$, where $I(W;S)$ denotes the Shannon mutual information between $W$ and $S$. \citet{bu2020tightening} proved a tighter expected generalization gap bound that depends on mutual information between $W$ and individual samples:
\begin{equation}
    \abs{\E_{P_{W,S}}\sbr{R(W) - r_S(W)}} \le \frac{1}{n}\sum_{i=1}^n\sqrt{2\sigma^2 I(W;Z_i)}.
    \label{eq:bu-et-al-bound}
\end{equation}
In fact, one can measure information with random subsets of examples of size $m$ and divide over $m$~\citep{negrea2019information, harutyunyan2021informationtheoretic}:
\begin{equation}
\abs{\E_{P_{W,S}}\sbr{R(W) - r_S(W)}} \le \E_{U}\sbr{\sqrt{\frac{2\sigma^2 I^{U}(W;Z_{U})}{m}}},
\label{eq:random-subset-bound}
\end{equation}
where $U$ is a uniformly random subset of $[n]\triangleq \mathset{1,2,\ldots,n}$, $Z_{U} \triangleq \mathset{Z_i : i \in U}$, and $I^Z(X; Y) \triangleq \KL{P_{X,Y|Z}}{P_{X|Z} P_{Y|Z}}$ denotes the disintegrated mutual information~\citep{negrea2019information}.
In this work bounds that depend on information quantities related to individual examples are called \emph{sample-wise} bounds.
It has been shown that the sample-wise bound of \eqref{eq:bu-et-al-bound} is the tightest among bounds like \eqref{eq:random-subset-bound}~\citep{negrea2019information,harutyunyan2021informationtheoretic}.
Moreover, in some cases the sample-wise bound is finite, while the bound of \citet{xu2017information} is infinite (see \citep{bu2020tightening}).

\paragraph{PAC-Bayes bounds} A practically more useful quantity is the average difference between population and empirical risks for a fixed training set $S$: $\E_{P_{W|S}}\sbr{R(W) - r_S(W)}$.
Typical PAC-Bayes bounds are of the following form: with probability at least $1-\delta$ over $P_S$,
\begin{equation}
    \E_{P_{W | S}}\sbr{R(W) - r_S(W)} \le B\rbr{\KL{P_{W|S}}{\pi}, n, \delta},
    \label{eq:typical-pac-bayes}
\end{equation}
where $\pi$ is a prior distribution over $\mathcal{W}$ that does not depend on $S$.
If we choose the prior distribution to be the marginal distribution of $W$ (i.e., $\pi = P_W=\E_{S}\sbr{Q_{W|S}}$), then the KL term in \eqref{eq:typical-pac-bayes} will be the integrand of the mutual information, as $I(W;S) = \E_{S}\sbr{\KL{P_{W|S}}{P_W}}$.
When the function $B$ depends on the KL term linearly, the expectation of the bound over $S$ will depend on the mutual information $I(W;S)$.
There are no known PAC-Bayes bounds where $B$ depends on KL divergences of form $\KL{P_{W|Z_i}}{P_W}$ or on sample-wise mutual information $I(W; Z_i)$.

\paragraph{Single-draw bounds}
In practice, even when the learning algorithm is stochastic, one usually draws a single hypothesis $W \sim Q_{W|S}$ and is interested in bounding the population risk of $W$. Such bounds are called single-draw bounds and are usually of the following form: with probability at least $1-\delta$ over $P_{W,S}$:
\begin{equation}
    R(W) - r_S(W) \le B\rbr{\der{Q_{W|S}}{\pi}, n,\delta},
\end{equation}
where $\pi$ is a prior distribution as in PAC-Bayes bounds.
When $\pi = P_W$ the single-draw bounds depends on the information density $\iota(W,S) = \der{Q_{W|S}}{P_W}$.
Single-draw bounds are the hardest to obtain and no sample-wise versions are known for them either.

\paragraph{Expected squared generalization gap bounds}
In terms of difficulty, expected generalization bounds are the easiest to obtain, and as indicated above, some of those bounds are sample-wise bounds.
If we consider a simple change of moving the absolute value inside: $\E_{P_{W,S}}\abs{R(W) - r_S(W)}$, then no sample-wise bounds are known. 
The same is true for the expected squared generalization gap $\E_{P_{W,S}}\rbr{R(W) - r_S(W)}^2$, for which the known bounds are of the following form~\citep{steinke2020reasoning, harutyunyan2021informationtheoretic, aminian2021information}:
\begin{equation}
    \E_{P_{W,S}}\sbr{\rbr{R(W) - r_S(W)}^2} \le \frac{I(W;S) + c}{n},
\end{equation}
where c is some constant.
From this result one can upper bound the expected absolute value of generalization gap using Jensen's inequality.

\subsection{Our contributions}
In \cref{sec:counterexample} we show that even for the expected squared generalization gap, sample-wise information-theoretic bounds are impossible.
The same holds for PAC-Bayes and single-draw bounds as well.
In \cref{sec:implications} we discuss the consequences for other information-theoretic generalization bounds.
Finally, in \cref{sec:m=2} we show that starting at subsets of size 2, there are expected squared generalization gap bounds that measure information between $W$ and a subset of examples.
This in turn implies that such PAC-Bayes and single-draw bounds are also possible, albeit they are not tight and high-probability bounds.

\section{A useful lemma}
We first prove a lemma that will be used in the proof of the main result of the next section.
\begin{lemma}
Consider a collection of bits $a_1,\ldots,a_{N_0 + N_1}$, such that $N_0$ of them are zero and the remaining $N_1$ of them are one.
We want to partition these numbers into $k=(N_0 + N_1)/n$ groups of size $n$ (assuming that $n$ divides $N_0 + N_1$).
Consider a uniformly random ordered partition $(A_1,\ldots,A_k)$.
Let $Y_i = \oplus_{a \in A_i} a$ be the parity of numbers of subset $A_i$.
For any given $\delta > 0$, there exists $N'$ such that when $\min\mathset{N_0,N_1} \ge N'$ then $\mathrm{Cov}\sbr{Y_i, Y_j} \le \delta\E\sbr{Y_1}^2$, for all $i, j \in [k], i \neq j$.
\label{lemma:cov}
\end{lemma}
\begin{proof}
By symmetry all random variables $Y_i$ are identically distributed.
Without loss of generality let's prove the result for the covariance of $Y_1$ and $Y_2$, which can be written as follows:
\begin{align*}
    \mathrm{Cov}\sbr{Y_1,Y_2} &= \E\sbr{Y_1 Y_2} - \E\sbr{Y_1}\E\sbr{Y_2}\\
    &\hspace{-4em}= P_{Y_1,Y_2}\rbr{Y_1 = 1, Y_2 = 1} - P(Y_1=1)P(Y_2=1).\numberthis\label{eq:cov_y1_y2}
\end{align*}
Consider the process of generating a uniformly random ordered partition by first picking $n$ elements for the first subset, then $n$ elements for the second subset, and so on.
In this scheme, the probability that both $Y_1=1$ and $Y_2=1$ equals to
\begin{align*}
\frac{1}{M} &\sum_{u=0}^{\floor{(n-1)/2}}\sum_{v=0}^{\floor{(n-1)/2}}\left[\Choose{N_1}{2u+1}\Choose{N_0}{n-2u-1}\times\right.\\
&\hspace{2em}\times\left.\Choose{N_1-2u-1}{2v+1}\Choose{N_0-n+2u+1}{n-2v-1}\right],\numberthis\label{eq:joint-prob}
\end{align*}
where {\footnotesize $M = \Choose{N_0+N1}{n}\Choose{N_0+N_1-n}{n}$}. Let $q_{u,v}$ be the $(u,v)$-th summand of \eqref{eq:joint-prob} divided by $M$.
On the other hand, the product of marginals $P(Y_1=1) P(Y_2=1)$ is equal to
\begin{align*}
\frac{1}{M'} &\sum_{u=0}^{\floor{(n-1)/2}}\sum_{v=0}^{\floor{(n-1)/2}}\left[ \Choose{N_1}{2u+1}\Choose{N_0}{n-2u-1}\right.\times\\
&\hspace{2em}\left.\times\Choose{N_1}{2v+1}\Choose{N_0}{n-2v-1}\right],\numberthis\label{eq:product-of-marginals-prob}
\end{align*}
where {\footnotesize $M' = \Choose{N_0 + N_1}{n}\Choose{N_0 + N_1}{n}$}. Let $q'_{u,v}$ be the $(u,v)$-th summand of \eqref{eq:product-of-marginals-prob} divided by $M'$.
Consider the ratio $q_{u,v}/q'_{u,v}$:
\begin{align*}
\resizebox{\columnwidth}{!}{
    $\displaystyle \frac{q_{u,v}}{q'_{u,v}} = \frac{\Choose{N_0 + N_1}{n}}{\Choose{N_0+N_1-n}{n}}\cdot\frac{\Choose{N_1-2u-1}{2v+1}}{\Choose{N_1}{2v+1}}\cdot\frac{\Choose{N_0-n+2u+1}{n-2v-1}}{\Choose{N_0}{n-2v-1}}.$
}
\end{align*}
By picking $N_0$ and $N_1$ large enough, we will make $q_{u,v}$ close to $q'_{u,v}$.
This is possible as for any fixed $n$, these 3 fractions converge to 1 as $\min\mathset{N_0, N_1}\rightarrow \infty$.
Therefore, for any $\delta > 0$ there exists $N'$ such that when $\min\mathset{N_0, N_1}\ge N'$ then $q_{u,v} \le (1+\delta) q'_{u,v}$.
This implies that  $P_{Y_1,Y_2}(Y_1=1,Y_2=1) \le (1+\delta) P(Y_1=1) P(Y_2=1)$. Combining this result with \eqref{eq:cov_y1_y2} proves the desired result.
\end{proof}

\section{A counterexample}\label{sec:counterexample}
\begin{theorem}
For any training set size $n=2^r$ and $\delta > 0$, there exists a finite input space $\ZZ$, a data distribution $P_Z$, a learning algorithm $Q_{W|S}$ with a finite hypothesis space $\mathcal{W}$, and a binary loss function $\ell : \mathcal{W} \times \ZZ \rightarrow \mathset{0, 1}$ such that 
\begin{itemize}
    \item[(a)] $\KL{Q_{W|S}}{P_{W}} \ge n-1$ with probability at least $1-\delta$,
    \item[(b)] $W$ and $Z_i$ are independent for each $i\in[n]$, 
    \item[(c)] $\E_{P_{W,S}}\sbr{R(W) - r_S(W)} = 0$,
    \item[(d)] $P_{W,S}\rbr{R(W) - r_S(W) \ge \frac{1}{4}} \ge \frac{1}{2}$.
\end{itemize}
\label{thm:main}
\end{theorem}
This result shows that there can be no meaningful sample-wise expected squared generalization bounds, as $I(W; Z_i) = 0$ while $\E_{P_{W,S}}\rbr{R(W) - r_S(W)}^2 \ge 1/16$.
Similarly, there can be no sample-wise PAC-Bayes or single-draw generalization bounds that depend on quantities such as $\KL{P_{W|Z_i}}{P_W}$, $I(W;Z_i)$ or $\iota(W, Z_i)$, as all of them are zero while with probability at least $1/2$ the generalization gap will be at least $1/4$.
The first property verifies that in order to make the generalization gap large $W$ needs to capture a significant amount of information about the training set.

The main idea behind the counterexample construction is to ensure that $W$ contains sufficient information about the whole training set but no information about individual examples.
This captured information is then used to make the losses on all training examples equal to each other, but possibly different for different training sets.
This way we induce significant variance of empirical risk.
Along with making the population risk to be roughly the same for all hypotheses, we ensure that the generalization gap will be large on average.
Satisfying the information and risk constraints separately is trivial, the challenge is in satisfying both at the same time.

\begin{proof}[Proof of \cref{thm:main}]
Let $n=2^r$ and $\ZZ$ be the set of all binary vectors of size $d$: $\ZZ=\mathset{0,1}^d$ with $d > r$. 
Let $N=2^d$ denote the cardinality of the input space.
We choose the data distribution to be the uniform distribution on $\ZZ$: $P_Z = \text{Uniform}(\ZZ)$.
Let the hypothesis set $\mathcal{W}$ be the set of all partitions of $\ZZ$ into subsets of size $n$:
\begin{equation}
\resizebox{0.91\columnwidth}{!}{
    $\displaystyle \mathcal{W} = \mathset{\{A_1,\ldots,A_{N/n}\} \mid |A_i|=n, \cup_i A_i=\ZZ, A_i\cap A_j = \emptyset}$.
    }
\end{equation}
When the training algorithm receives a training set $S$ that contains duplicate examples, it outputs a hypothesis from $\mathcal{W}$ uniformly at random.
When $S$ contains no duplicates, then the learning algorithm outputs a uniformly random hypothesis from the set $\mathcal{W}_S \subset \mathcal{W}$ of hypotheses/partitions that contain $S$ as a partition subset.
Formally,
\begin{equation}
    Q_{W|S} = \left\{\begin{array}{ll}
    \mathrm{Uniform}(\mathcal{W}), & \text{if $S$ has duplicates},\\
    \mathrm{Uniform}\rbr{\mathcal{W}_S}, & \text{otherwise},
    \end{array}\right.
\end{equation}
where $\mathcal{W}_S \triangleq \mathset{\mathset{A_1,\ldots,A_{N/n}} \in \mathcal{W} \mid \exists i \text{ s.t. } A_i = S}$.
Let $\rho(n,d)$ be the probability of $S$ containing duplicate examples.
By picking $d$ large enough we can make $\rho$ as small as needed.

Given a partition $w=\{A_1,\ldots,A_{2^d/n}\}\in\mathcal{W}$ and an example $z\in\ZZ$, we define $[z]_w$ to be the subset $A_i \in w$ that contains $z$.
Given a set of examples $A \subset \ZZ$, we define $\oplus^{(2)}(A)$ to be xor of all bits of all examples of $A$:
\begin{equation}
    \oplus^{(2)}(A) = \oplus_{(a_1,\ldots,a_d)\in A}\rbr{\oplus_{i=1}^d a_i}.
\end{equation}
Finally, we define the loss function as follows:
\begin{equation}
    \ell(w,z) = \oplus^{(2)}\rbr{[z]_w} \in \mathset{0,1}.
\end{equation}

Let $W\sim Q_{W|S}$. Let us verify now that properties (a)-(d) listed in the statement of \cref{thm:main} hold.

\textbf{(a).} By symmetry, the marginal distribution $P_W = \E_S\sbr{Q_{W|S}}$ will be the uniform distribution over $\mathcal{W}$. With probability $1-\rho(n,d)$, the training set $S$ has no duplicates and $Q_{W|S} = \text{Uniform}(\mathcal{W}_S)$. In such cases, the support size of $Q_{W|S}$ is equal to $ \frac{(N-n)!}{(n!)^{N/n-1}(N/n-1)!}$, while the support size of $P_W$ is always equal to $\frac{N!}{(n!)^{N/n}(N/n)!}$.
Therefore, 
\begin{align*}
    \KL{Q_{W|S}}{P_W} &= \log\rbr{\frac{\abs{\mathrm{supp}(P_W)}}{\abs{\mathrm{supp}(Q_{W|S})}}}\\
    &\hspace{-3em}=\log\rbr{\frac{(N-n+1)(N-n+2)\cdots N}{n!(N/n)}}\\
    &\hspace{-3em}\ge\log\rbr{\frac{n(N-n+1)^n}{n^n N}}\\
    &\hspace{-3em}=n\log\rbr{N-n+1} - (n-1)\log n - \log N\\
    &\hspace{-3em}\approx (n-1) d\log 2 - n\log n. &&\text{\hspace{-7.5em}(for a suff. large $d$)}
\end{align*}
For large $d$, the quantity above is approximately $(n-1)d\log 2$, which is expected 
as knowledge of any $Z_i$ ($d$ bits) along with $W$ is enough to reconstruct the training set $S$ that has $nd$ bits of entropy.
To satisfy the property (a), we need to pick $d$ large enough to make $\rho(n,d)<\delta$ and $(n-1)d\log 2 - n\log n \ge n-1$.

\textbf{(b).} Consider any $z\in\ZZ$ and any $i \in [n]$. Then
    \begin{align*}
        P_{W|Z_i=z}(W=w) &= \frac{P_W(W=w)P_{Z_i|W=w}(Z_i=z)}{P_Z(Z=z)}\\
        &= P_W(W=w),
    \end{align*}
    where the second equality follows from the fact that conditioned on a fixed partition $w$, because of the symmetry, there should be no difference between probabilities of different $Z_i$ values.

\textbf{(c).} With probability $1-\rho(n,d)$,
    \begin{align*}
        r_S(W) &= \frac{1}{n}\sum_{i=1}^n \ell(W,Z_i)\\
        &= \frac{1}{n}\sum_{i=1}^n \oplus^{(2)}([Z_i]_W)\\
        &= \frac{1}{n}\sum_{i=1}^n \oplus^{(2)}(S)\\
        &= \oplus^{(2)}(S) \in \{0,1\}\numberthis\label{eq:emp_risk}.
    \end{align*}
    Furthermore, 
    \begin{align*}
        R(W) &= \E_{P_Z}\sbr{\ell(W,Z)} = \frac{1}{N/n} \sum_{A \in W} \oplus^{(2)}(A).
        \numberthis\label{eq:risk}
    \end{align*}
    Given \eqref{eq:emp_risk} and \eqref{eq:risk}, due to symmetry  $\E_{P_{W,S}}\sbr{r_S(W)}=1/2$ and $\E_{P_{W,S}}\sbr{R(W)} = 1/2$. Hence, the expected generalization gap is zero: $\E_{P_{W,S}}\sbr{R(W) - r_S(W)} = 0$.

\textbf{(d).} Consider a training set $S$ that has no duplicates and let $W=\mathset{S, A_1,\ldots,A_{N/n-1}} \sim Q_{W|S}$.
The population risk can be written as follows:
\begin{align}
    R(W) = \underbrace{\frac{n}{N} \oplus^{(2)}(S)}_{\le n/N} + \frac{n}{N} \sum_{i=1}^{N/n-1} \underbrace{\oplus^{(2)}(A_i)}_{\triangleq Y_i}.
\end{align}
Consider the set $\ZZ \setminus S$. Let $N_0$ be the number of examples in this set with parity 0 and $N_1$ be the number of examples with parity 1. We use \cref{lemma:cov} to show that $Y_i$ are almost pairwise independent. Formally, for any $\delta'>0$ there exists $N'$ such that when $\min\mathset{N_0, N_1} \ge N'$, we have that $\mathrm{Cov}\sbr{Y_i, Y_j} \le \delta'\E\sbr{Y_1}^2$, for all $i, j \in [N/n-1], i\neq j$. Therefore,
\begin{align*}
    \mathrm{Var}\sbr{\frac{n}{N} \sum_{i=1}^{N/n-1}Y_i} &= \frac{n^2}{N^2} \rbr{\frac{N}{n}-1} \mathrm{Var}\sbr{Y_1}\\
    &+\frac{n^2}{N^2}\rbr{\frac{N}{n}-1}\rbr{\frac{N}{n}-2}\mathrm{Cov}\sbr{Y_1, Y_2}\\
    &\le \frac{n}{4N} + \delta'.
\end{align*}
By Chebyshev's inequality
\begin{small}
\begin{align*}
P\rbr{\abs{\frac{n}{N}\sum_{i=1}^{N/n-1} Y_i - \frac{n}{N}\rbr{\frac{N}{n}-1}\E\sbr{Y_1}} \ge t} \le \frac{\frac{n}{4N} + \delta'}{t^2}.
\end{align*}
\end{small}
Furthermore, $\E\sbr{Y_1}\rightarrow 1/2$ as $d \rightarrow \infty$. Therefore, we can pick a large enough $d$, appropriate $\delta'$ and $t$ to ensure that with at least 0.5 probability over $P_{W,S}$, $R(W) \in [1/4, 3/4]$ and $r_S(W) \in \mathset{0,1}$.
\end{proof}

\section{Implications for other bounds}\label{sec:implications}
Information-theoretic generalization bounds have been improved or generalized in many ways.
A few works have proposed to use other types of information measures and distances between distributions, instead of Shannon mutual information and Kullback-Leibler divergence respectively~\citep{esposito2021generalization,aminian2021jensen,rodriguez2021tighter}.
In particular, ~\citet{rodriguez2021tighter} derived expected generalization gap bounds that depend on the Wasserstein distance between $P_{W|Z_i}$ and $P_W$. \citet{aminian2021jensen} derived similar bounds but that depend on sample-wise Jensen-Shannon information $I_{JS}(W; Z_i)\triangleq\mathrm{JS}\rbr{P_{W,Z_i}\left\vert\right\vert P_W  P_{Z_i}}$ or on lautum information $L(W;Z_i) \triangleq \KL{P_W  P_{Z_i}}{P_{W,Z_i}}$.
\citet{esposito2021generalization} derive bounds on probability of an event in a joint distribution $P_{X,Y}$ in terms of the probability of the same event in the product of marginals distribution $P_X P_Y$ and an information measure between $X$ and $Y$ (Sibson’s $\alpha$-mutual information, maximal leakage, $f$-mutual information) or a divergence between $P_{X,Y}$ and $P_X P_Y$ (R\'{e}nyi $\alpha$-divergences, $f$-divergences, Hellinger divergences).
They note that one can derive in-expectation generalization bounds from these results. These results will be sample-wise if one starts with $X=W$ and $Y=Z_i$ and then takes an average over $i\in[n]$.
The property (b) of the counterexample implies that there are no PAC-Bayes, single-draw, or expected squared generalization bounds for the aforementioned information measures or divergences, as all of them will be zero when $\forall z\in\ZZ, P_W = P_{W|Z_i=z}$.

PAC-Bayes bounds have been improved by comparing  population and empirical risks differently (instead of just subtracting them)~\citep{langford2001bounds,Germain2009PACBayesianLO, Rivasplata2020PACBayesAB}. The property (d) of the counterexamples implies that these improvements will not make sample-wise PAC-Bayes bounds possible, as the changed distance function will be at least constant when $r_S(W) \in\mathset{0,1}$ while $R(W) \in [1/4,3/4]$.

Another way of improving information-theoretic bounds is to use the random-subsampling setting introduced by \citet{steinke2020reasoning}.
In this setting one considers $2n$ i.i.d. samples from $P_Z$ grouped into $n$ pairs: $\tilde{Z} \in \ZZ^{n\times 2}$.
A random variable $J \sim \text{Uniform}(\mathset{0,1}^n)$, independent of $\tilde{Z}$, specifies which example to select from each pair to form the training set $S = (\tilde{Z}_{i,J_i})_{i=1}^n$.
\citet{steinke2020reasoning} proved that if $\ell(w,z) \in [0,1], \forall w \in \WW, z \in \ZZ$, then the expected generalization gap can be bounded as follows:
\begin{equation}
    \abs{\mathbb{E}_{S,W}\sbr{R(W) - r_S(W)}} \le \sqrt{\frac{2}{n} I(W; J \mid \tilde{Z})}.
\end{equation}
This result was improved in many works, leading to the following sample-wise bounds~\citep{haghifam2020sharpened, harutyunyan2021informationtheoretic,Galvez2021OnRS,Zhou2021IndividuallyCI}:
\begin{small}
\begin{align}
    \abs{\mathbb{E}_{S,W}\sbr{R(W) - r_S(W)}} &\le \frac{1}{n}\sum_{i=1}^n \E_{\tilde{Z}_i}\sbr{\sqrt{2 I^{\tilde{Z}_i}(W; J_i)}}\label{eq:random-subsampling-sawple-wise-strong},\\
    \abs{\mathbb{E}_{S,W}\sbr{R(W) - r_S(W)}}&\le \frac{1}{n}\sum_{i=1}^n \E_{\tilde{Z}}\sbr{\sqrt{2 I^{\tilde{Z}}(W; J_i)}}\label{eq:random-subsampling-sawple-wise-weak},
\end{align}
\end{small}
where $\tilde{Z}_i=(\tilde{Z}_{i,1}, \tilde{Z}_{i,2})$ is the $i$-th row of $\tilde{Z}$.
Given a partition $W \sim Q_{W|S}$ and two examples $\tilde{Z}_i$, one cannot tell which of the examples was in the training set because of the symmetry.
Hence, the counterexample implies that expected squared, PAC-Bayes, and single draw generalization bounds that depend on quantities like $I^{\tilde{Z}_i}(W; J_i)$ cannot exist.
However, if we consider the weaker sample-wise bounds of \eqref{eq:random-subsampling-sawple-wise-weak}, then the knowledge of $\tilde{Z}$ helps to reveal the entire $J$ at once with high probability.
This can be done by going over all possible choices of $J$ and checking whether $\tilde{Z}_J = (\tilde{Z}_{i, J_i})$ belongs to partition $W$.
This will be true for the true value of $J$, but will be increasingly unlikely for all other values of $J$ as $n$ and $d$ are increased.
In fact, we derive an expected squared generalization gap that depends on terms $I^{\tilde{Z}}(W; J_i)$ (see \cref{thm:g2-weak-sample-wise-bound} of the Appendix).
Therefore, the counterexample is a discrete case where the bound of \eqref{eq:random-subsampling-sawple-wise-strong} is much better than the weaker bound of \eqref{eq:random-subsampling-sawple-wise-weak}.
It is also a case where $I(W; Z_i) \ll I(W; J_i \mid \tilde{Z})$ (i.e., CMI bounds are not always better).

Finally, another way of improving information-theoretic bounds is to use evaluated (e-CMI) or functional (f-CMI) conditional mutual information ~\citep{steinke2020reasoning, harutyunyan2021informationtheoretic}.
Similarly to the functional CMI bounds of \cite{harutyunyan2021informationtheoretic}, one can derive the following expected generalization gap bounds (see \cref{thm:ecmi-sample-wise-exp-gen-gap}):
\begin{small}
\begin{align}
    \small
    \abs{\mathbb{E}\sbr{R(W) - r_S(W)}} &\le \frac{1}{n}\sum_{i=1}^n \sqrt{2 I(\ell(W, \tilde{Z}_i); J_i)},\label{eq:random-subsampling-sawple-wise-ECMI-strongest}\\
    \abs{\mathbb{E}\sbr{R(W) - r_S(W)}} &\le \frac{1}{n}\sum_{i=1}^n \E_{\tilde{Z}_i}\sbr{\sqrt{2 I^{\tilde{Z}_i}(\ell(W, \tilde{Z}_i); J_i)}},\label{eq:random-subsampling-sawple-wise-ECMI-strong}\\
    \abs{\mathbb{E}\sbr{R(W) - r_S(W)}} &\le \frac{1}{n}\sum_{i=1}^n \E_{\tilde{Z}}\sbr{\sqrt{2 I^{\tilde{Z}}(\ell(W, \tilde{Z}_i); J_i)}}\label{eq:random-subsampling-sawple-wise-ECMI-weak},
\end{align}
\end{small}
where $\ell(W, \tilde{Z}_i) \in \mathset{0,1}^2$ is the pair of losses on the two examples of $\tilde{Z}_i$.
In the case of the counterexample bounds of \eqref{eq:random-subsampling-sawple-wise-ECMI-strongest} and \eqref{eq:random-subsampling-sawple-wise-ECMI-strong} will be zero as one cannot guess $J_i$ knowing losses $\ell(W,\tilde{Z}_i)$ and possibly also $\tilde{Z}_i$.
This rules out the possibility of such sample-wise expected squared, PAC-Bayes and single-draw generalization bounds.
Unlike the case of the weaker CMI bound of \eqref{eq:random-subsampling-sawple-wise-weak}, the weaker e-CMI bound of \eqref{eq:random-subsampling-sawple-wise-ECMI-weak} convergences to zero in the case of the counterexample as $d\rightarrow \infty$.
Therefore, the counterexample is a discrete case where the sample-wise e-CMI bound of \eqref{eq:random-subsampling-sawple-wise-ECMI-weak} can be much stronger than the sample-wise CMI bound of \eqref{eq:random-subsampling-sawple-wise-weak}.

\section{The case of $m=2$}\label{sec:m=2}
In \cref{sec:intro} we mentioned that there are expected generalization bounds that are based on information contained in size $m$ subsets of examples.
In \cref{sec:counterexample} we showed that there can be no expected squared generalization bounds with $m=1$.
In this section we show that expected squared generalization bounds are possible for any $m \ge 2$.
For brevity, let $G^{(2)} \triangleq \E_{P_{W,S}}\sbr{\rbr{R(W) - r_S(W)}^2}$ denote 
expected squared generalization gap.

\begin{theorem} Assume $\ell(w,z) \in [0, 1]$. Let $W \sim Q_{W|S}$. Then
\begin{equation}
    G^{(2)} \le \frac{1}{n} + \frac{1}{n^2}\sum_{i \neq k} \sqrt{2 I\rbr{W; Z_i, Z_k}}.
\end{equation}
\label{thm:m=2}
\end{theorem}

\begin{proof}
We have that
\begin{small}
\begin{align*}
    G^{(2)} &= \E_{P_{W,S}}\sbr{\rbr{\frac{1}{n}\sum_{i=1}^n\rbr{\ell(W,Z_i) - R(W)}}^2}\\
    &= \underbrace{\frac{1}{n^2}\sum_{i=1}^n \E_{P_{W,Z_i}}\sbr{\rbr{\ell(W,Z_i) - R(W)}^2}}_{\le 1/n} \\
    &\quad+ \frac{1}{n^2}\sum_{i \neq k}\underbrace{\E\sbr{\rbr{\ell(W,Z_i) - R(W)}\rbr{\ell(W,Z_k) - R(W)}}}_{C_{i,k}}.
\end{align*}
\end{small}
For bounding the $C_{i,k}$ terms we use the Lemma 1 of \citet{xu2017information} with $X = W$, $Y = (Z_i, Z_k)$, and $f(X,(Y_1,Y_2)) = \rbr{\ell(X, Y_1) - R(X)}\rbr{\ell(X,Y_2)-R(X)}$.
As $f(X,Y)$ is 1-subgaussian under $P_X P_Y$, by the lemma
\begin{equation*}
    \abs{\E_{P_{X,Y}}\sbr{f(X,Y)} - \E_{P_X P_Y}\sbr{f(X,Y)}} \le \sqrt{2 I(X;Y)},
\end{equation*}
which translates to
\begin{small}
\begin{align*}
    &\abs{C_{i,k} -  \underbrace{\E_{P_{W}P_{Z_i,Z_k}}\sbr{\rbr{\ell(W, Z_i) - R(W)}\rbr{\ell(W,Z_k)-R(W)}}}_{\bar{C}_{i,k}}} \\
    &\quad\quad\le \sqrt{2 I(W;Z_i,Z_k)}.\numberthis
\end{align*}
\end{small}
It is left to notice that $\bar{C}_{i,k}=0$, as for any $w$ the factors $\rbr{\ell(w, Z_i) - R(w)}$ and $\rbr{\ell(w,Z_k)-R(w)}$ are independent and have zero mean. 
\end{proof}

\begin{remark}
Similar CMI and e-CMI bounds are also possible (see \cref{thm:g2-pairwise-CMI}).
\end{remark}

\begin{corollary}
Let $U_m$ be a uniformly random subset of $[n]$ of size $m$, independent from $W$ and $S$. Then
\begin{equation}
    G^{(2)} \le \frac{1}{n} +  2\E_{U_m}\sbr{\sqrt{\frac{I^{U_m}(W; S_{U_m})}{m}}}.
\end{equation}
\label{corollary:m=k}
\end{corollary}
\begin{proof}
\citet{harutyunyan2021informationtheoretic} showed that for any $m\in[n-1]$,
\begin{equation*}
\resizebox{\columnwidth}{!}{
$\displaystyle \E_{U_m}\sbr{\sqrt{\frac{1}{m} I^{U_m}(W; S_{U_m})}} \le \E_{U_{m+1}}\sbr{\sqrt{{\frac{1}{m+1} I^{U_{m+1}}(W; S_{U_{m+1}})}}}.$
}
\end{equation*}
Therefore, for any $m=2,\ldots,n$, starting with \cref{thm:m=2},
\begin{align*}
G^{(2)} &\le \frac{1}{n} + \frac{1}{n^2}\sum_{i \neq k} \sqrt{2 I\rbr{W; Z_i, Z_k}}\\
&\le \frac{1}{n} + 2 \E_{U_2}\sbr{\sqrt{\frac{I^{U_2}(W; S_{U_2})}{2}}}\\
&\le \frac{1}{n} +  2 \E_{U_m}\sbr{\sqrt{\frac{I^{U_m}(W; S_{U_m})}{m}}}.
   \end{align*}
\end{proof}

At $m=n$ this bound is weaker that the bound derived in \cite{harutyunyan2021informationtheoretic}, which depends on $I(W;S)/n$ rather than $\sqrt{I(W;S)/n}$.
We leave improving the bound of \cref{thm:m=2} for future work.
Nevertheless, \cref{corollary:m=k} shows that it is \emph{possible} to bound the expected \emph{squared} generalization gap with quantities that involve mutual information terms between $W$ and subsets of examples of size $m$, where $m \ge 2$ (unlike the case of $m=1$).
Possibility of bounding expected squared generalization gap with $m\ge 2$ information terms makes it possible for single-draw and PAC-Bayes bounds as well.
The simplest way is to use Markov's inequality, even though it will not give high probability bounds.

Finally, it is worth to mention that the bound of \cref{thm:m=2} holds for higher order moments of generalization gap too, as for $[0,1]$-bounded loss functions 
\begin{equation*}
    \E_{P_{W,S}}\sbr{\rbr{R(W) - r_S(W)}^k} \le \E_{P_{W,S}}\sbr{\rbr{R(W) - r_S(W)}^2},
\end{equation*}
for any $k \ge 2, k \in \mathbb{N}$.

\section{Conclusion}
In the counterexample presented in Sec.~\ref{sec:counterexample} the empirical risk is sometimes larger than the population risk, which is  rare in practice.
In fact, if empirical risk is never larger than population risk, then $\E\sbr{\abs{R(W)-r_S(W)}}$ reduces to $\E\sbr{R(W)-r_S(W)}$, implying existence of sample-wise bounds.
Furthermore, the constructed learning algorithm intentionally captures only high-order information about samples.
This suggests, that sample-wise generalization bounds might be possible if we consider specific learning algorithms.

\bibliographystyle{IEEEtranN}
\bibliography{main.bib}

\clearpage
\onecolumn
\appendix

In this appendix we present the full statements and the proofs of the results that were omitted from the main text due to space constraints.

\begin{lemma}[Lemma 1 of \citet{xu2017information}]  Let $(X, Y)$ be a pair of random variables with joint distribution $P_{X, Y}$. 
If $f(x, y)$ is a measurable function such that $\E_{P_X P_Y}\sbr{f(X, Y)}$ exists and $f(X,Y)$ is $\sigma$-subgaussian under $P_X P_Y$, then
\begin{align}
\abs{\E_{P_{X, Y}}\sbr{f(X, Y)} - \E_{P_X P_Y}\sbr{f(X, Y}} \le \sqrt{2 \sigma^2 I(X; Y)}.
\end{align}
\label{lemma:xu-raginsky-lemma}
\end{lemma}

\begin{theorem}
In the random-subsampling setting, let $W \sim Q_{W|S}$. If $\ell(w,z)\in[0,1]$, then
\begin{align}
    \abs{\E_{P_{W,S}}\sbr{R(W) - r_S(W)}} \le \min\bigg\{&\frac{1}{n}\sum_{i=1}^n \sqrt{2 I\rbr{\ell(W, \tilde{Z}_i); J_i}},\\
    &\frac{1}{n}\sum_{i=1}^n \E_{\tilde{Z}_i}\sqrt{2 I^{\tilde{Z}_i}\rbr{\ell(W, \tilde{Z}_i); J_i}},\\
    &\frac{1}{n}\sum_{i=1}^n \E_{\tilde{Z}}\sqrt{2 I^{\tilde{Z}}\rbr{\ell(W, \tilde{Z}_i); J_i}}\bigg\}.
\end{align}
\label{thm:ecmi-sample-wise-exp-gen-gap}
\end{theorem}
\begin{proof}
Let $\tilde{\Lambda} \in [0, 1]^{n\times 2}$ be the losses on examples of $\tilde{Z}$:
\begin{equation}
    \tilde{\Lambda}_{i, c} = \ell(W, \tilde{Z}_{i, c}),\ \ \forall i\in [n], c\in\mathset{0,1}.
\end{equation}
Let $\bar{J} \triangleq (1-J_1,\ldots,1 - J_n)$ be the negation of $J$.
We have that
\begin{align}
    \E_{P_{W,S}}\sbr{r_S(W) - R(W)} &= \frac{1}{n}\sum_{i=1}^n \E\sbr{\ell(W,Z_i) - \E_{Z'\sim P_Z}\ell(W,Z')}\\
    &= \frac{1}{n}\sum_{i=1}^n \E\sbr{\ell(W,Z_{i, J_i}) - \ell(W, Z_{i, \bar{J}_i})}\\
    &=\frac{1}{n}\sum_{i=1}^n \E\sbr{\tilde{\Lambda}_{i,J_i} - \tilde{\Lambda}_{i, \bar{J}_i}}.
\end{align}
If we use \cref{lemma:xu-raginsky-lemma} with $X = \tilde{\Lambda}_i$, $Y=J_i$, and $f(X,Y) = \tilde{\Lambda}_{i,J_i} - \tilde{\Lambda}_{i, \bar{J}_i}$, we get that
\begin{align}
    \abs{\E\sbr{\tilde{\Lambda}_{i,J_i} - \tilde{\Lambda}_{i, \bar{J}_i}} - \E_{\tilde{\Lambda}_i}\E_{J_i}\sbr{\tilde{\Lambda}_{i,J_i} \tilde{\Lambda}_{i, \bar{J}_i}}} \le \sqrt{2 I(\tilde{\Lambda}_i, J_i)}
\end{align}
This proves the first part of the theorem, as $\E_{\tilde{\Lambda}_i}\E_{J_i}\sbr{\tilde{\Lambda}_{i,J_i} \tilde{\Lambda}_{i, \bar{J}_i}} = 0$.
The second part can be proven by first conditioning on $\tilde{Z}_i$:
\begin{equation}
    \E\sbr{\tilde{\Lambda}_{i,J_i} - \tilde{\Lambda}_{i, \bar{J}_i}} = \E_{\tilde{Z}_i}\E_{\tilde{\Lambda}_i, J_i |\tilde{Z}_i}\sbr{\tilde{\Lambda}_{i,J_i} - \tilde{\Lambda}_{i, \bar{J}_i}},
\end{equation}
and then applying the lemma to upper bound $\abs{\E_{\tilde{\Lambda}_i, J_i |\tilde{Z}_i}\sbr{\tilde{\Lambda}_{i,J_i} - \tilde{\Lambda}_{i, \bar{J}_i}}}$ with $\sqrt{2I^{\tilde{Z_i}}(\tilde{\Lambda}_i;J_i)}$.
Finally, the third part can be proven by first conditioning on $\tilde{Z}$:
\begin{equation}
    \E\sbr{\tilde{\Lambda}_{i,J_i} - \tilde{\Lambda}_{i, \bar{J}_i}} = \E_{\tilde{Z}}\E_{\tilde{\Lambda}_i, J_i |\tilde{Z}}\sbr{\tilde{\Lambda}_{i,J_i} - \tilde{\Lambda}_{i, \bar{J}_i}},
\end{equation}
and then applying the lemma to upper bound $\abs{\E_{\tilde{\Lambda}_i, J_i |\tilde{Z}}\sbr{\tilde{\Lambda}_{i,J_i} - \tilde{\Lambda}_{i, \bar{J}_i}}}$ with $\sqrt{2I^{\tilde{Z}}(\tilde{\Lambda}_i;J_i)}$.
\end{proof}

\begin{remark} As $\tilde{Z}$ and $J_i$ are independent, $I\rbr{\ell(W, \tilde{Z}_i); J_i} \le \E_{\tilde{Z}_i}\sbr{I^{\tilde{Z}_i}\rbr{\ell(W, \tilde{Z}_i); J_i}} \le \E_{\tilde{Z}}\sbr{I^{\tilde{Z}}\rbr{\ell(W, \tilde{Z}_i); J_i}}$. However, if we consider expected square root of disintegrated mutual informations (as \cref{thm:ecmi-sample-wise-exp-gen-gap}), then this relation might not be true.
\end{remark}

\begin{theorem}
In the random-subsampling setting, let $W \sim Q_{W|S}$. If $\ell(w,z)\in[0,1]$, then
\begin{equation}
    \E_{P_{W,S}}\sbr{\rbr{R(W) - r_S(W)}^2} \le \frac{5}{2n} + \frac{2}{n}\sum_{i=1}^n \E_{\tilde{Z}}\sqrt{2 I^{\tilde{Z}}(W; J_i)}.
\end{equation}
\label{thm:g2-weak-sample-wise-bound}
\end{theorem}
\begin{proof}
Let $\bar{J} \triangleq (1-J_1,\ldots,1 - J_n)$ be the negation of $J$. We have that
\begin{align}
\E_{P_{W,S}}\sbr{\rbr{R(W) - r_S(W)}^2} &= \E_{W,S}\sbr{\rbr{\frac{1}{n}\sum_{i=1}^n \ell(W,Z_i) - \E_{Z'\sim P_Z}\ell(W,Z')}^2}\\
&\hspace{-8em}=\E_{\tilde{Z},J,W}\sbr{\rbr{\frac{1}{n}\sum_{i=1}^n \ell(W,\tilde{Z}_{i, J_i}) - \E_{Z'\sim P_Z}\ell(W,Z')}^2}\\
&\hspace{-8em}=\E_{\tilde{Z},J,W}\sbr{\rbr{\frac{1}{n}\sum_{i=1}^n \ell(W,\tilde{Z}_{i, J_i}) - \E_{Z'\sim P_Z}\ell(W,Z') + \frac{1}{n}\sum_{i=1}^n \ell(W,\tilde{Z}_{i, \bar{J}_i}) - \frac{1}{n}\sum_{i=1}^n \ell(W,\tilde{Z}_{i, \bar{J}_i})}^2}\\
&\hspace{-8em}\le 2\underbrace{\E_{\tilde{Z},J,W}\sbr{\rbr{\frac{1}{n}\sum_{i=1}^n \ell(W,\tilde{Z}_{i, J_i}) - \frac{1}{n}\sum_{i=1}^n \ell(W,\tilde{Z}_{i, \bar{J}_i})}^2}}_{B}\\
&\hspace{-8em}\quad+2\E_{\tilde{Z},J,W}\sbr{\rbr{\frac{1}{n}\sum_{i=1}^n \ell(W,\tilde{Z}_{i, \bar{J}_i}) - \E_{Z'\sim P_Z}\ell(W,Z')}^2}&&\hspace{-14em}\text{(as $\E\sbr{(U +V)^2}\le 2\E\sbr{U^2} + 2\E\sbr{V^2}$)}\\
&\hspace{-8em}= 2B + 2\E_{\tilde{Z}, J, W}\sbr{\rbr{\frac{1}{n}\sum_{i=1}^n \rbr{\ell(W,\tilde{Z}_{i,\bar{J}_i}) - \E_{Z'\sim P_Z}\ell(W,Z')}}^2}\\
&\hspace{-8em}= 2B + 2\E_{W}\E_{\tilde{Z},J | W}\sbr{\rbr{\frac{1}{n}\sum_{i=1}^n \rbr{\ell(W,\tilde{Z}_{i,\bar{J}_i}) - \E_{Z'\sim P_Z}\ell(W,Z')}}^2}.\label{eq:test-concentration}
\end{align}
For any fixed $w\in\mathcal{W}$, the terms $\ell(W, \tilde{Z}_{i,\bar{J_i}})$ are independent of each other under $P_{\tilde{Z}, J | W=w}$.
Furthermore, $W$ and $\tilde{Z}_{i, \bar{J}_i}$ are independent.
Therefore, the average in \eqref{eq:test-concentration} is an average of $n$ i.i.d. random variables with zero mean. The variance of this average is at most $\frac{1}{4n}$.
Hence,
\begin{equation}
    \E_{P_{W,S}}\sbr{\rbr{R(W) - r_S(W)}^2} \le 2B + \frac{1}{2n}.\label{eq:bounding-g2-first-step}
\end{equation}
Let us bound $B$ now:
\begin{align}
B &= \E_{\tilde{Z},J,W}\sbr{\rbr{\frac{1}{n}\sum_{i=1}^n\rbr{ \ell(W,\tilde{Z}_{i, J_i}) - \ell(W,\tilde{Z}_{i, \bar{J}_i})}}^2}\\
    &= \frac{1}{n^2}\sum_{i=1}^n \E_{\tilde{Z},J,W}\sbr{\rbr{\ell(W,\tilde{Z}_{i, J_i}) - \ell(W,\tilde{Z}_{i, \bar{J}_i})}^2} \\
    &\quad+ \frac{1}{n^2}\sum_{i \neq k}\E_{\tilde{Z},J,W}\sbr{\rbr{\ell(W,\tilde{Z}_{i, J_i}) - \ell(W,\tilde{Z}_{i, \bar{J}_i})}\rbr{\ell(W,\tilde{Z}_{k, J_k}) - \ell(W,\tilde{Z}_{k, \bar{J}_k})}}\\
    &\le \frac{1}{n} + \E_{\tilde{Z},J,W}\sbr{\frac{1}{n}\sum_{i=1}^n\rbr{\rbr{\ell(W,\tilde{Z}_{i, J_i}) - \ell(W,\tilde{Z}_{i, \bar{J}_i})}\frac{1}{n}\sum_{k \neq i} \rbr{\ell(W,\tilde{Z}_{k, J_k}) - \ell(W,\tilde{Z}_{k, \bar{J}_k})}}}.\label{eq:bounding-g2-second-step}
\end{align}
Let us consider a fixed $i\in[n]$. Then
\begin{align*}
    &\E_{\tilde{Z},J,W}\sbr{\rbr{\ell(W,\tilde{Z}_{i, J_i}) - \ell(W,\tilde{Z}_{i, \bar{J}_i})}\frac{1}{n}\sum_{k \neq i} \rbr{\ell(W,\tilde{Z}_{k, J_k}) - \ell(W,\tilde{Z}_{k, \bar{J}_k})}}\\
    &\quad\quad=\E_{\tilde{Z}}\E_{J_i, W|\tilde{Z}}\underbrace{\E_{J_{-i} | J_i, W,  \tilde{Z}}\sbr{\rbr{\ell(W,\tilde{Z}_{i, J_i}) - \ell(W,\tilde{Z}_{i, \bar{J}_i})}\frac{1}{n}\sum_{k \neq i} \rbr{\ell(W,\tilde{Z}_{k, J_k}) - \ell(W,\tilde{Z}_{k, \bar{J}_k})}}}_{f(W,J_i,\tilde{Z})}.\label{eq:bounding-g2-third-step}\numberthis
\end{align*}
Note that $f(w,j_i,\tilde{z}) \in [-1, +1]$ for all $w\in\mathcal{W}, j \in \mathset{0,1}^n$, and $\tilde{z} \in \ZZ^{n\times 2}$. Therefore, by \cref{lemma:xu-raginsky-lemma}, for any value of $\tilde{Z}$
\begin{equation}
\E_{W,J_i|\tilde{Z}}\sbr{f(W, J_i, \tilde{Z})} \le \sqrt{2 I^{\tilde{Z}}(W; J_i)} + \E_{W\mid \tilde{Z}}\E_{J_i\mid \tilde{Z}}\sbr{f(W, J_i, \tilde{Z})}.\label{eq:bounding-g2-dv}
\end{equation}
It is left to notice that for any $w\in\mathcal{W}$, $\E_{J_i | \tilde{Z}}\sbr{f(w, J_i | \tilde{Z})} = 0$, as under $P_{J_i | \tilde{Z}}$ the term $\ell(w,\tilde{Z}_{i, J_i})-\ell(w,\tilde{Z}_{i,\bar{J}_i})$ has zero mean.
Therefore, \eqref{eq:bounding-g2-dv} reduces to
\begin{equation}
\E_{W,J_i|\tilde{Z}}\sbr{f(W, J_i, \tilde{Z})} \le \sqrt{2 I^{\tilde{Z}}(W; J_i)}.\label{eq:bounding-g2-dv-reduced}
\end{equation}
Putting together \eqref{eq:bounding-g2-first-step}, \eqref{eq:bounding-g2-second-step}, \eqref{eq:bounding-g2-third-step} and \eqref{eq:bounding-g2-dv-reduced}, we get that
\begin{equation}
    \E_{P_{W,S}}\sbr{\rbr{R(W) - r_S(W)}^2} \le \frac{5}{2n} + \frac{2}{n}\sum_{i=1}^n \E_{\tilde{Z}}\sqrt{2 I^{\tilde{Z}}(W; J_i)}.
\end{equation}
\end{proof}

\begin{theorem}
In the random-subsampling setting, let $W \sim Q_{W|S}$. If $\ell(w,z)\in[0,1]$, then
\begin{align}
\E_{P_{W,S}}\sbr{\rbr{R(W) - r_S(W)}^2} &\le \frac{5}{2n} + \frac{2}{n^2}\sum_{i \neq k} \sqrt{2I(\ell(W,\tilde{Z}_i),\ell(W,\tilde{Z}_k); J_i, J_k)},\label{eq:g2-pairwise-ecmi-strongest}\\
\E_{P_{W,S}}\sbr{\rbr{R(W) - r_S(W)}^2} &\le \frac{5}{2n} + \frac{2}{n^2}\sum_{i \neq k} \E\sqrt{2I^{\tilde{Z}_i,\tilde{Z}_k}(\ell(W,\tilde{Z}_i),\ell(W,\tilde{Z}_k); J_i, J_k)},\label{eq:g2-pairwise-ecmi-strong}\\
\E_{P_{W,S}}\sbr{\rbr{R(W) - r_S(W)}^2} &\le \frac{5}{2n} + \frac{2}{n^2}\sum_{i \neq k} \E\sqrt{2I^{\tilde{Z}}(\ell(W,\tilde{Z}_i),\ell(W,\tilde{Z}_k); J_i, J_k)}\label{eq:g2-pairwise-ecmi-weak},
\end{align}
and
\begin{align}
\E_{P_{W,S}}\sbr{\rbr{R(W) - r_S(W)}^2} &\le \frac{5}{2n} + \frac{2}{n^2}\sum_{i \neq k} \E\sqrt{2I^{\tilde{Z}_i,\tilde{Z}_k}(W; J_i, J_k)},\label{eq:g2-pairwise-cmi-strong}\\
\E_{P_{W,S}}\sbr{\rbr{R(W) - r_S(W)}^2} &\le \frac{5}{2n} + \frac{2}{n^2}\sum_{i \neq k} \E\sqrt{2I^{\tilde{Z}}(W; J_i, J_k)}\label{eq:g2-pairwise-cmi-weak}.
\end{align}
\label{thm:g2-pairwise-CMI}
\end{theorem}
\begin{proof}
It is enough to prove only \eqref{eq:g2-pairwise-ecmi-strongest}, \eqref{eq:g2-pairwise-ecmi-strong}, and \eqref{eq:g2-pairwise-ecmi-weak}. Inequality \eqref{eq:g2-pairwise-cmi-strong} can be derived from \eqref{eq:g2-pairwise-ecmi-strong} using data processing inequality, while \eqref{eq:g2-pairwise-cmi-weak} can be derived from \eqref{eq:g2-pairwise-ecmi-weak}.

As in the proof of \cref{thm:g2-weak-sample-wise-bound}, 
\begin{align}
    \E_{P_{W,S}}\sbr{\rbr{R(W) - r_S(W)}^2} \le 2\underbrace{\E_{\tilde{Z},J,W}\sbr{\rbr{\frac{1}{n}\sum_{i=1}^n \ell(W,\tilde{Z}_{i, J_i}) - \frac{1}{n}\sum_{i=1}^n \ell(W,\tilde{Z}_{i, \bar{J}_i})}^2}}_{B} + \frac{1}{2n}.
\end{align}
where $\bar{J} \triangleq (1-J_1,\ldots,1 - J_n)$ and 
\begin{equation}
   B \le \frac{1}{n} + \frac{1}{n^2}\sum_{i \neq k}\underbrace{\E_{\tilde{Z},J,W}\sbr{\rbr{\ell(W,\tilde{Z}_{i, J_i}) - \ell(W,\tilde{Z}_{i, \bar{J}_i})}\rbr{\ell(W,\tilde{Z}_{k, J_k}) - \ell(W,\tilde{Z}_{k, \bar{J}_k})}}}_{C_{i,k}}.
\end{equation}
Let $\tilde{\Lambda} \in [0, 1]^{n\times 2}$ be the losses on examples of $\tilde{Z}$:
\begin{equation}
    \tilde{\Lambda}_{i, c} = \ell(W, \tilde{Z}_{i, c}),\ \ \forall i\in [n], c\in\mathset{0,1}.
\end{equation}
Then we can write
\begin{align}
    C_{i,k} = \E_{\tilde{\Lambda}_i,\tilde{\Lambda}_k,J_i, J_k}\sbr{\rbr{\tilde{\Lambda}_{i, J_i} - \tilde{\Lambda}_{i, \bar{J}_i}}\rbr{\tilde{\Lambda}_{k, J_k} - \tilde{\Lambda}_{k, \bar{J}_k}}}.
\end{align}
and use Lemma 1 of \citet{xu2017information} to arrive at:
\begin{align}
    C_{i,k} &\le \E_{\tilde{\Lambda}_i,\tilde{\Lambda}_k}\E_{J_i, J_k}\sbr{\rbr{\tilde{\Lambda}_{i, J_i} - \tilde{\Lambda}_{i, \bar{J}_i}}\rbr{\tilde{\Lambda}_{k, J_k} - \tilde{\Lambda}_{k, \bar{J}_k}}} + \sqrt{2 I(\tilde{\Lambda}_i, \tilde{\Lambda}_k; J_i, J_k)}\\
    &=\sqrt{2 I(\tilde{\Lambda}_i, \tilde{\Lambda}_k; J_i, J_k)},
\end{align}
where the last equality holds as for any fixed value of $(\tilde{\Lambda}_i, \tilde{\Lambda}_k)$ the terms $\rbr{\tilde{\Lambda}_{i, J_i} - \tilde{\Lambda}_{i, \bar{J}_i}}$ and $\rbr{\tilde{\Lambda}_{k, J_k} - \tilde{\Lambda}_{k, \bar{J}_k}}$ are independent and have zero mean under $P_{J_i, J_k}$.

To derive \eqref{eq:g2-pairwise-ecmi-strong}, one can condition on $\tilde{Z}_i, \tilde{Z}_k$:
\begin{equation}
    C_{i,k} = \E_{\tilde{Z}_i,\tilde{Z}_k}\E_{\tilde{\Lambda}_i,\tilde{\Lambda}_k,J_i, J_k | \tilde{Z}_i,\tilde{Z}_k}\sbr{\rbr{\tilde{\Lambda}_{i, J_i} - \tilde{\Lambda}_{i, \bar{J}_i}}\rbr{\tilde{\Lambda}_{k, J_k} - \tilde{\Lambda}_{k, \bar{J}_k}}},
\end{equation}
and then apply \cref{lemma:xu-raginsky-lemma} for the inner expectation.
Similarly, \eqref{eq:g2-pairwise-ecmi-weak} can be derived by conditioning on $\tilde{Z}$ and then applying \cref{lemma:xu-raginsky-lemma}.
\end{proof}

\begin{remark}
Wit the expectation inside the square root, inequalities \eqref{eq:g2-pairwise-ecmi-strongest}, \eqref{eq:g2-pairwise-ecmi-strong}, and \eqref{eq:g2-pairwise-ecmi-weak} would be in non-decreasing order. With the expectation outside the square root, this relation might not be true.
\end{remark}
\end{document}